\documentclass{article}
\usepackage{ijcai16}
\usepackage{times}
\usepackage{comment}
\usepackage{algorithm}
\usepackage{graphicx}
\usepackage[normalem]{ulem}
\usepackage{subfigure}
\usepackage{xcolor}
\usepackage{algpseudocode}
\usepackage{amssymb}
\usepackage{amsthm}
\usepackage[fleqn]{amsmath}
\newcounter{bar}
\usepackage{url}

\newcommand{\conmcmc}{{\sc Con-MCMC}}
\newcommand{\G}{{\mathcal{G}}}
\newcommand{\X}{{\mathcal{X}}}
\newtheorem{defn}{Definition}[section]
\newtheorem{theorem}{Theorem}
\newtheorem{lemma}[theorem]{Lemma}

\newtheorem{thm}{Theorem}[section]

\title{Contextual Symmetries in Probabilistic Graphical Models\thanks{This StarAI-16 paper is a close revision of [Anand {\em et al.}, 2016].}}
\author{Ankit Anand \\Indian Institute of Technology, Delhi \\ankit.anand@cse.iitd.ac.in \And Aditya Grover \\
Stanford University \\ adityag@cs.stanford.edu  \And Mausam \and Parag Singla \\ 
Indian Institute of Technology, Delhi \\
\{mausam,parags\}@cse.iitd.ac.in}

\begin{document}

\maketitle
\nocite{anand&al16}
\begin{abstract}
 An important approach for efficient inference in probabilistic graphical models exploits symmetries among objects in the domain. Symmetric variables (states) are collapsed into meta-variables (meta-states) and inference algorithms are run over the lifted graphical model instead of the flat one. Our paper extends existing definitions of symmetry by introducing the novel notion of {\em contextual symmetry}. Two states that are not globally symmetric, can be contextually symmetric under some specific assignment to a subset of variables, referred to as the {\em context variables}. Contextual symmetry subsumes previous symmetry definitions and can represent a large class of symmetries not representable earlier. We show how to compute contextual symmetries by reducing it to the problem of graph isomorphism. We extend previous work on exploiting symmetries in the MCMC framework to the case of contextual symmetries. Our experiments on several domains of interest demonstrate that exploiting contextual symmetries can result in significant computational gains.
\end{abstract}

\section{Introduction}
\label{sec:intro}

An important approach for efficient inference in probabilistic graphical models exploits symmetries in the underlying domain. It is especially useful for statistical relational learning models such as Markov logic networks \cite{richardson&domingos06}, which exhibit repeated sub-structures -- many objects are indistinguishable from each other and their associated relations have identical probability distributions. {\em Lifted inference} algorithms (see \cite{kimmig&al15} for a survey) exploit this phenomenon by grouping symmetric states (variables) into meta-states (meta-variables) and performing inference in this reduced (lifted) graphical model.

Early approaches to lifted inference devised first order extensions of propositional inference algorithms. These include approaches for lifting exact inference algorithms such as variable elimination \cite{poole03,braz&al05}, weighted model counting \cite{gogate&domingos11}, knowledge compilation~\cite{broeck&al11}, as well as lifting approximate algorithms such as belief propagation \cite{singla&domingos08,kersting&al09,singla&al14}, Gibbs sampling~\cite{venugopal&gogate12} and importance sampling \cite{gogate&al12}. In all these approaches, the lifting technique is tied to the specific algorithm being considered. More recently, another line of work \cite{jha&al10,bui&al13,niepert&broeck14,sarkhel&al14,kopp&al15} has started looking at the notion of symmetry independent of the inference technique. In several cases, these symmetries are compactly represented using permutation groups. The computed symmetries have been used downstream for lifting existing algorithms such as variational inference \cite{bui&al13},  (integer) linear programming \cite{noessner&al13,mladenov&al14}, and Markov chain Monte Carlo (MCMC) \cite{niepert12,broeck&niepert15}, which is our focus.

A key shortcoming of existing algorithms is that they only identify and exploit sets of variables (states) that are symmetric {\em unconditionally}. Our goal is to extend the notion of symmetries to {\em contextual} symmetries, sets of states that are symmetric under a given context (variable-value assignment). Our proposal is inspired by the extension of conditional independence to context-sensitive independence \cite{boutilier&al96}, and analogously extends unconditional symmetries to contextual. As our first contribution, we develop a formal framework to define contextual symmetries. We also present an algorithm to compute contextual symmetries by reducing the problem to graph isomorphism.

Figure \ref{fig:Example-Contextual} illustrates an example of contextual symmetries. A couple A and B may like to go to a romantic movie. They are somewhat less (equally) likely to go alone compared to when they go together. 
However if the movie is a thriller, A may be less interested in going by herself, but B may not change his behavior. Hence, A and B are symmetric to each other if the movie is romantic, but not symmetric if the movie is a thriller. We will call the A and B contextually symmetric conditioned on the movie being romantic. 

\begin{figure}\label{fig:Example}
\subfigure[]{\includegraphics[width=0.16\textwidth]{{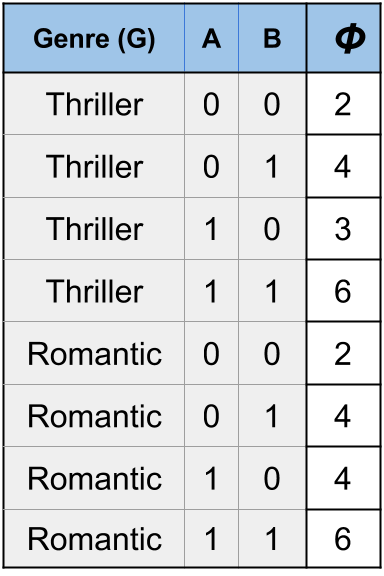}}\label{fig:Example-Contextual}}
\includegraphics[width=0.15\textwidth]{{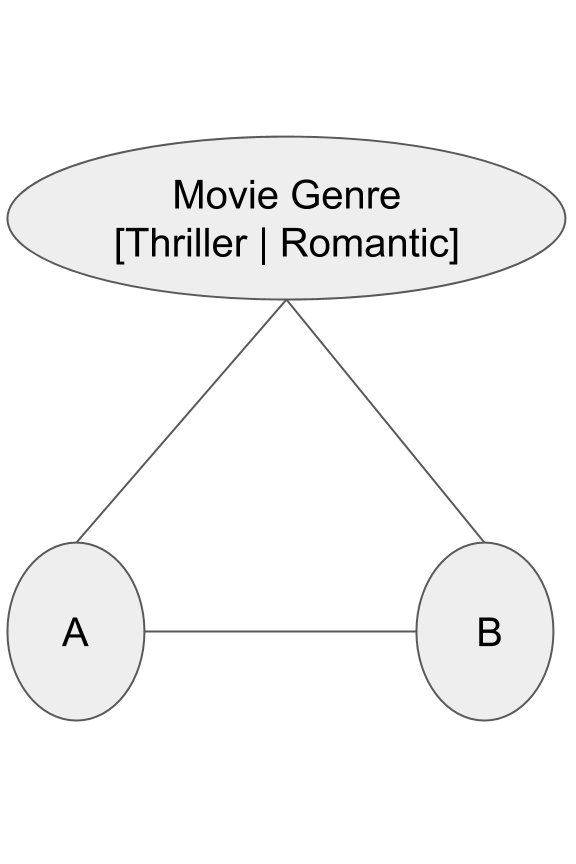}}
\subfigure[]{\includegraphics[width=0.16\textwidth]{{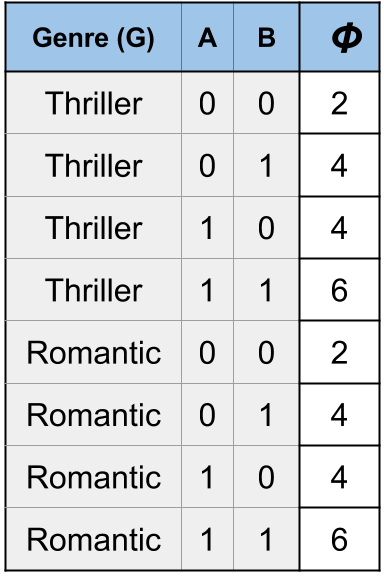}}\label{fig:Example-Orbital}}
\caption{Illustration of (a) Contextual Symmetry (with Genre=``Romantic'')  (b) Orbital Symmetry in the Movie Network.}
\end{figure}

Finally, our paper extends the line of work on Orbital MCMC \cite{niepert12} -- a state-of-the-art approach to exploit unconditional symmetries in a generic MCMC framework. 
Orbital MCMC achieved reduced mixing times compared to Gibbs sampling in domains where such symmetries exist. We design \conmcmc, an algorithm that uses contextual symmetries within the MCMC framework. Our experiments demonstrate that on various interesting domains (relational and propositional), where contextual symmetries may be present, \conmcmc\ can yield substantial gains compared to Orbital MCMC and Gibbs Sampling. We also release a reference implementation of \conmcmc\ sampler for wider use.\footnote{\url{https://github.com/dair-iitd/con-mcmc}}
%We have released a reference implementation of \conmcmc sampler ~\footnote{\url{https://github.com/dair-iitd/con-mcmc}}.

\section{Background}

Let $\mathcal{X}= \{X_1, \ldots, X_n\}$ be a finite set of discrete
random variables. For ease of exposition, we consider Boolean random variables although our analysis extends more generally to n-ary random variables. We denote $s\in\{0,1\}^n$ to represent a {\em state}. A {\em graphical model} $\mathcal{G}$ over $\mathcal{X}$ can be
represented as the set of pairs $\{f_k,w_k\}_{k=1}^{m}$ where 
$f_k$ is a formula (feature) over a subset of variables in $\mathcal{X}$ 
and $w_k$ is the associated weight \cite{koller&friedman09}. This is 
the representation used in several existing models such as Markov 
logic networks \cite{domingos&lowd09}.

\subsection{Symmetries in Graphical Models} 
Some states may be symmetric; thus, they will have the same joint probability. This fact can be exploited in inference algorithms. To define symmetries, we make use of the formalism of {\em automorphism groups}, which are generic representations for symmetries between any set of objects. Automorphism groups over graphical models are defined using another algebraic structure called a {\em permutation group}.

A permutation $\theta$ is a bijection from the
set $\mathcal{X}$ onto itself. We use $\theta(X)$ to 
denote the application of $\theta$ to the element $X \in \mathcal{X}$.
We also overload $\theta$ by denoting $\theta(s)$ as the permutation of a state $s$ in which 
each component random variable $X$ is permuted using $\theta(X)$. 

A permutation group $\Theta$ is a set of permutations 
which contains the identity element, has a unique inverse for every element 
in the set, and is closed under composition operator. 

Following previous work \cite{niepert12} we define symmetries and automorphism groups in graphical models as follows:
 
\begin{defn}
A {\bf symmetry of a graphical model} $\mathcal{G}$ over the set $\cal{X}$ is
represented as the permutation $\theta$ of the variables in $\mathcal{X}$  
that maps $\mathcal{G}$ back on to itself, i.e. results in the 
same set of weighted formulas.
\end{defn}

\begin{defn} 
An {\bf automorphism group of a graphical model} $\mathcal{G}$ is defined as the permutation group ($\Theta$) over $\mathcal{G}$ such that each $\theta \in \Theta$ is a symmetry of $\mathcal{G}$.
\end{defn}
This definition of an automorphism group of a graphical model is analogous to that of an automorphism group of an edge-colored graph, where variables in $\mathcal{G}$  act as vertices in graph, features act as edges (or hyperedges) and weights act as colors on edges. We next define  an orbit of a state. 
\begin{defn}
The {\bf orbit ($\Gamma$) of a state} $s$ under the automorphism group $\Theta$ is defined as all the states that can be reached by applying a symmetry $\theta \in \Theta$ on the variables of $s$, i.e., $\Gamma_{\Theta}(s)=\{s' \in \{0,1\}^n \mid \exists
\theta \in \Theta$ s.t. $\theta(s) = s'\}$.
\end{defn}
Henceforth, we will refer to these unconditional symmetries of a graphical model as {\em orbital symmetries}. Let $P(s)$ be the probability distribution defined by model $\mathcal{G}$ over the states.
\begin{thm}
If $\Theta$ is an automorphism group of $\mathcal{G}$, then $\forall s'\in \Gamma_\Theta(s):~ P(s) = P(s')$, i.e. orbital symmetries of a graphical
model are probability preserving transformations.
\end{thm}

Therefore, the automorphism group $\Theta$ for a graphical model is also referred to as an automorphism group for the underlying probability distribution. The symmetries of a graphical models as defined above can be obtained via solving a graph automorphism problem. Though the problem is not known to be in P or NP-complete,\footnote{A quasipolynomial time algorithm \cite{babai-arxiv2015} has been proposed recently for the related graph isomorphism problem which remains to be verified.} efficient solutions 
can be obtained using the software such as Saucy and Nauty
~\cite{saucy,nauty}.

\subsection{Markov chain Monte Carlo}\label{sec:background-mcmc}
Markov chain Monte Carlo (MCMC) is a popular approach for approximate inference in graphical models. A Markov chain is set up over the state space such that its stationary distribution is the same as the underlying probability distribution. An {\em Orbital Markov chain} \cite{niepert12} exploits the orbital symmetries of a model by setting up a Markov chain combining the original MCMC moves with orbital moves. Let $M_o$ denote an orbital Markov chain and $M$ be the corresponding original Markov chain. Then, given the current state $s^{(t)}$, the next state $s^{(t+1)}$ in $M_o$ is sampled as follows:
\begin{itemize}
\item {\it original move}: sample an intermediate state $s'^{(t+1)}$ from $s^{(t)}$ based on the transition probability in $M$.
\item {\it orbital move}: sample a next state $s^{(t+1)}$ uniformly at random from the orbit $\Gamma_\Theta(s'^{(t+1)})$.
\end{itemize}
Orbital MCMC converges to the same stationary distribution as the original MCMC and is shown to have significantly faster convergence properties.

\section{Contextual Symmetries}
\label{sec:framework}
Our work proposes the novel notion of contextual symmetries -- symmetries that only hold under a given context. We now extend the definitions of the previous section to their contextual counterparts. First we define a context.

\vspace{0.5ex}
\begin{defn} 
A {\bf context} $C$ is a partial assignment, i.e., a set of pairs $(X_i,x_i)$, where $X_i \in \mathcal{X}$ and $x_i \in \{0,1\}$, and no $X_i$ is repeated in the set. 

\end{defn}

For example, in Figure~\ref{fig:Example}, 
we can define a context (Genre, ``Romantic"). We refer to a context as a {\em 
single variable context} if there is only one element in the context 
set. 

We say that a variable $X_i$ appears in a context if there
is a pair $(X_i,x_i) \in C$. Given a context $C$, we will use 
$\mathcal{X}_C$ to denote the subset of variables of $\mathcal{X}$
which appear in $C$. We will use $\bar{\mathcal{X}}_{{C}}$ to denote
the complement of this set. Given a state $s$, we will use $\zeta_{X_i}(s)$ to denote 
the value of $X_i$ in state $s$. We say that a state $s$ is {\em consistent}
with the context $C$ iff $\forall (X_i,x_i) \in C$ we have $x_i=\zeta_{X_i}(s)$. 

In order to define contextual automorphism, we will need to define the notion of a reduced model.
\vspace{0.5ex}
\begin{defn}  
Given a graphical model $\mathcal{G}= \{f_k,w_k\}_{k=1}^{m}$, and a context $C$, the {\bf reduced model} $\mathcal{G}^{r}_{C}$ is defined as the new graphical model obtained by substituting $X_i=x_i$ in each formula $f_k$ for all $(X_i,x_i)\in C$ and keeping the original weights $w_k$.
\end{defn}

Note that $\G^{r}_C$ is defined over the set 
$\bar{\mathcal{X}}_{{C}}$. As an example, if the model is 
represented by the formulas \{($P \vee Q$, $w_1$), ($R \vee Q 
\vee S$, $w_2$)\}, the reduced model under the single variable context $\{(R,0)\}$ 
will be \{($P \vee Q$, $w_1$) and ($Q \vee S$, $w_2$)\}. In the 
factored form representation, reduction by a context corresponds to 
fixing the values of the context variables in the potential table. E.g., in 
Figure~\ref{fig:Example} given the context "Romantic", we 
reduce the factor to the bottom four rows of the potential
table where Genre has value "Romantic".
We are now ready to define a contextual symmetry of
a graphical model.

\vspace{0.5ex}
\begin{defn}\label{def:consym}
A {\bf contextual symmetry of a graphical model} $\G$ under context $C$ is represented as a permutation $\theta$ of variables in $\mathcal{X}$ s.t. a) $\theta(X_i) = X_i$,
$\forall X_i \in \mathcal{X}_C$ i.e. variables in the context
are mapped to themselves, and b) $\exists$ an orbital symmetry 
$\theta^r$ of the reduced model $\G^{r}_{C}$ 
such that $\theta(X_i)=\theta^r(X_i)$ $\forall
X_i \in \bar{\mathcal{X}}_{{C}}$, i.e. mapping of the remaining
variables defines an orbital symmetry of the reduced graphical model under context $C$.
\end{defn}

For example, in Figure \ref{fig:Example}, let a permutation $\theta^*$ be: $\theta^*(G)=G, \theta^*(A)=B, \theta^*(B)=A$. $\theta^*$ is a contextual symmetry under the context (Genre, ``Romantic''), but not under the context (Genre, ``Thriller'').

\vspace{0.5ex}
\begin{defn}
A {\bf contextual automorphism group of a graphical model} $\G$ under context $C$ is defined as a permutation group ($\Theta_C$) over $\G$, such that each $\theta\in\Theta_C$ is a contextual symmetry of $\G$ under context $C$.
\end{defn}

\vspace{0.5ex}
\begin{defn}
The {\bf contextual orbit of a state} $s$ under the contextual automorphism group $\Theta_C$ (given the context $C$) is the set of those states which are consistent with $C$ and can be reached by applying $\theta \in \Theta_C$ to $s$, i.e., $\Gamma_{\Theta_C}(s)=\{s' \in \{0,1\}^n \mid \exists \theta \in \Theta_C$ s.t. $\theta(s) = s' \bigwedge \forall (X_i,x_i) \in C, \zeta_{X_i}(s')=x_i\}$. 
\end{defn}
Note that $s$ must be consistent with $C$ for it to have a non-empty contextual orbit.
Analogous to orbital symmetries, contextual symmetries are also probability preserving.

\begin{thm}\label{thm:probequi}
A contextual symmetry $\theta$ of $\G$ under context $C=\{(X_i,x_i)\}$ is probability preserving $\left(P(s)=P(\theta(s))\right)$, as long as 
$s$ is consistent with $C$.
\end{thm}

\subsection{Relationship with Related Concepts}
\label{sec:rel}

The set of contextual symmetries subsumes that of orbital symmetries  -- any orbital symmetry is a contextual symmetry under a null context $\emptyset$. The two notions are even more related, as the following two lemmas show. Let $\mathcal{X}_\theta^I$ be the set of variables that map onto itself in a permutation $\theta$, i.e. $\forall X\in \mathcal{X}_\theta^I: \theta(X)=X$.

\begin{lemma}
An orbital symmetry $\theta$ is a contextual symmetry under a context $C$ if $\X_C\subseteq \X_\theta^I$.
\end{lemma}

\begin{lemma}
Let $V\subseteq \X$.  If a permutation $\theta$ is a contextual symmetry of $\G$ under all possible contexts $C_i$ where $\X_{C_i}=V$, then $\theta$ is an orbital symmetry of $\G$.
\end{lemma}

We now distinguish the notions of context and contextual symmetries from two other related concepts.  First, a context is different from evidence.
External information in the form of evidence modifies the underlying
distribution represented by the graphical model. In contrast, a context  has no effect on the underlying distribution.

Second, it might be tempting to confuse contextually symmetric states with
contextually independent states \cite{boutilier&al96}. In the example of Figure \ref{fig:Example}(a) given Genre=``Thriller'', A and B are contextually independent, i.e., probability of A does not change depending on B. For this context A and B are non-symmetric. For Genre=``Romantic'' A and B are symmetric but not independent. 

Finally, in Section \ref{sec:related}, we discuss the relationship between contextual symmetries and the recent notion of conditional decomposability \cite{niepert&broeck14}. 

\subsection{Computing contextual symmetries}
\label{subsec:compute}

Computing contextual symmetries for $\mathcal{G}$ under a context $C$ is equivalent to computing orbital symmetries on the reduced model $\mathcal{G}_C^r=\{f_k^r, w_k^r\}_{k=1}^{m^r}$. To compute orbital symmetries we adapt the procedure from Niepert \shortcite{niepert12}. Following Niepert, we describe the construction when each $f_k^r$ is a clause, though it can be extended to the more general case.

Niepert's procedure creates a colored graph, with two nodes corresponding to every variable (one each for the positive and negative state), and one node for every formula $f_k^r$. Edges exist between the positive and negative states of every variable, and also between the formula nodes and the variable nodes (either positive or negative) appearing in the formula. Finally, colors are assigned to nodes based on the following criteria: (a) every positive variable node is assigned a common color, (b) every negative variable node is assigned a different common color, and (c) every unique formula weight $w_k^r$ is assigned a new color. The formula nodes $f_k^r$ inherit the color associated with their weight $w_k^r$. 

This color graph is then passed through a graph isomorphism solver (e.g., Saucy), which 
computes the automorphism group for $\mathcal{G}_C^r$. This is equivalent to computing contexual automorphism group for $\G$ under $C$:

\begin{thm}
The automorphism group for the color graph of the reduced graphical model $\mathcal{G}_C^r$ along with an identity mapping of the context variables gives a contextual automorphism group of $\G$ under $C$. 
\end{thm}

Note that in case we have any evidence $E$ available, the reduced model over which we induce a colored graph corresponds to $\mathcal{G}_{C\cup E}^r$. This is in contrast with original Niepert's procedure, where evidence nodes are not removed from the color graph and instead act as additional formulas for the original graphical model with infinity weights. This elimination of evidence nodes helps discover many more symmetries in the corresponding color graph while still preserving correctness. For example, if the model is represented by formulas \{($P \vee R$, $w_1$), ($Q$, $w_1$)\}, and evidence is ($\neg{R}$), $P$ and $Q$ become symmetric only if $R$ is eliminated from the color graph, and not in Niepert's procedure.

\section{Contextual MCMC} \label{section:mcmc}
We now extend the Orbital MCMC algorithm from Section \ref{sec:background-mcmc} so that it can exploit contextual symmetries; our algorithm is named \conmcmc, and is parameterized by $\alpha\in[0,1)$. Orbital MCMC  reduces mixing times over original MCMC,  because it can easily transition between high probability states falling in the same orbit, which may otherwise be separated by low probability regions. Unfortunately, as Figure \ref{fig:Example} demonstrates, a domain may have little orbital symmetry, but still important contextual symmetry. \conmcmc($\alpha$) exploits these for inference.

We are given a set of context variables $V \subset \mathcal{X}$ (more on this later). Let $\mathcal{C}_V$ denote the set of all possible contexts involving all the variables in $V$. Overloading the notation, we will use $\mathcal{C}_V(s)$ to denote the (unique) context in $\mathcal{C}_V$ consistent with state $s$. We compute contextual symmetries $\Theta_{C}$ under each context $C \in \mathcal{C}_V$ using the algorithm from Section \ref{sec:framework}. We are also given an original regular Markov chain $M$ that converges to the desired probability distribution $\pi(s)$. \conmcmc($\alpha$) runs a {\em Contextual Markov Chain} $M_{con}(\alpha)$ that samples a state $s^{(t+1)}$ from $s^{(t)}$ as follows: 

\begin{enumerate}
\item {\em Gibbs-orig move:} We sample an intermediate state $s'^{(t+1)}$ from the current state $s^{(t)}$ as: \\
\quad \quad (a) with probability $\alpha$ (Gibbs): flip a random context variable in $s^{(t)}$ using Gibbs transition probability. \\
\quad \quad (b) with probability $1$-$\alpha$ (original): make the move from $s^{(t)}$ based on the transition probability in $M$. 
\item {\em con-orbital move: } Let $C = \mathcal{C}_V(s'^{(t+1)})$ be the context consistent with $s'^{(t+1)}$. Let $\Gamma_{\Theta_{C}}(s'^{(t+1)})$ denote the contextual orbit of $s'^{(t+1)}$ under the context $C$. Sample a state $s^{(t+1)}$ uniformly at random from $\Gamma_{\Theta_{C}}(s'^{(t+1)})$. 
\end{enumerate}

When $\alpha=0$ our algorithm reduces to a direct extension of Orbital MCMC, where in second step, we sample uniformly from a contextual orbit instead of the original orbits. In the more interesting case of $\alpha > 0$, we enable the Markov chain to move more freely between different contexts using a Gibbs flip over the context variables. This Gibbs transition helps us carry over the effect of symmetries exploited under one context (via the orbital moves in step 2) to others. This can be especially useful  when symmetries are unevenly distributed across multiple contexts (as also confirmed by our experiments). 

In order to sample a state uniformly at random from a contextual orbit, we use the product replacement algorithm \cite{pak00} as described and used by Niepert \shortcite{niepert12}. Recall that since we are working with contextual permutations, the context variables are mapped to themselves and we are guaranteed to not change the context. Next, we show that \conmcmc($\alpha$) converges to the desired stationary distribution $\pi(s)$. We need the following lemma.
\begin{lemma}\label{lem:M_go_stationary}
Let $M_1$ and $M_2$ be two Markov chains defined over a finite state space $S$ with transition probability functions $P_1$ and $P_2$, respectively, such that $\pi(s)$ is a stationary distribution for both $P_1,P_2$, i.e., $\pi(s)=\sum_{r \in S} \pi(r)*P_i(r \rightarrow s)$, $i \in \{1,2\}$. Further, let $M_2$ be regular. Then, the Markov chain $M'$ with the transition function $P'(s \rightarrow r) = \alpha*P_1(s \rightarrow r) + (1-\alpha)*P_2(s \rightarrow r)$ is also regular and has a unique stationary distribution $\pi(s)$ for $\alpha \in [0,1)$.
\end{lemma}
Let $M^{go}(\alpha)$ refer to the family of Markov chains constructed using only step 1 of our algorithm i.e. no orbital moves. $M$ is regular with the stationary distribution $\pi(s)$. Further, each individual Gibbs flip over a variable satisfies stationarity with respect to the underlying distribution $\pi(s)$ \cite{koller&friedman09}. Hence, using Lemma~\ref{lem:M_go_stationary}, $M^{go}(\alpha)$ is regular with $\pi(s)$ as a stationary distribution.
\begin{theorem}
The family of contextual Markov chains $M^{con}(\alpha)$  constructed using \conmcmc($\alpha$)  converges to the stationary distribution of original Markov chain $M$ for any choice of context variables $V$ and $\alpha \in [0,1)$.
\end{theorem}
\begin{proof}
Let $\pi(s)$ be the stationary distribution of $M$. Since $M^{go}(\alpha)$ is regular it is easy to see that $M^{con}(\alpha)$ is also regular (there is always a non-zero probability of coming back to the same state in an orbital move). Therefore, $M^{con}(\alpha)$ converges to a unique stationary distribution. Then, we only need to show that $\pi(s)$ is the stationary distribution of $M^{con}(\alpha)$.
Let $S=\{0,1\}^{n}$ denote the set all of all the states. For $r,s \in S$, let $P^{go}[\alpha](s\rightarrow r)$ and $P^{con}[\alpha](s\rightarrow r)$ represent the transition probability functions of $M^{go}[\alpha]$ and $M^{con}[\alpha]$, respectively. In order to show that $M^{con}[\alpha]$ also converges to $\pi(s)$, we need to show that:

\begin{equation}
\pi(s) = \sum_{r \in S} \pi(r)  P^{con}[\alpha] (r \rightarrow s) 
\end{equation}
The RHS of the above equation can be written as:
\begin{eqnarray*}
  &=& \sum_{r \in S} \pi(r) \sum_{s' \in \Gamma_{\Theta_{\mathcal{C}_{V}(s)}}(s)} P^{go}[\alpha] (r \rightarrow s')  \frac{1}{|\Gamma_{\Theta_{\mathcal{C}_{V}(s)}}(s)|} \\
       &=& \sum_{r \in S} \sum_{s' \in \Gamma_{\Theta_{\mathcal{C}_{V}(s)}}(s)} \pi(r) P^{go}[\alpha] (r \rightarrow s')  \frac{1}{|\Gamma_{\Theta_{\mathcal{C}_{V}(s)}}(s)|}\\
\end{eqnarray*}
\begin{eqnarray*}
&=& \sum_{s' \in \Gamma_{\Theta_{\mathcal{C}_{V}(s)}}} \left[ \sum_{r \in S}  \pi(r)P^{go}[\alpha] (r \rightarrow s') \right]  \frac{1}{|\Gamma_{\Theta_{\mathcal{C}_{V}(s)}}(s)|}\\
&=& \sum_{s' \in \Gamma_{\Theta_{\mathcal{C}_{V}(s)}}} \pi(s')  \frac{1}{|\Gamma_{\Theta_{\mathcal{C}_{V}(s)}}(s)|}\\
&=& \sum_{s' \in \Gamma_{\Theta_{\mathcal{C}_{V}(s)}}} \pi(s)  \frac{1}{|\Gamma_{\Theta_{\mathcal{C}_{V}(s)}}(s)|}\\
       &=& \pi(s)
\end{eqnarray*}
Here, recall that $\Theta_{\mathcal{C}_{V}(s)}$ denotes the contextual automorphism group for the (unique) context $\mathcal{C}_V(s)$ consistent with the state $s$, and $\Gamma_{\Theta_{\mathcal{C}_{V}(s)}}(s)$ denotes the corresponding orbit. Step 1 above follows from the definition of contextual orbital move. Step 4 follows from the stationarity of $M^{go}[\alpha]$. Step 5 follows from the fact that all the states in the same contextual orbit have the same probability (Theorem \ref{thm:probequi}).
\end{proof}

\section{Experimental Evaluation}
\label{sec:expt}

Our experiments evaluate the use of contextual symmetries for faster inference in graphical models. We compare our approach against Orbital MCMC, which is the only available algorithm that exploits symmetries in a general MCMC framework. We also compare with vanilla Gibbs sampling, which does not exploit any symmetries. We implement \conmcmc($\alpha$) as an extension of the original Orbital MCMC implementation\footnote{\url{ https://code.google.com/archive/p/lifted-mcmc/}} available in the GAP language \cite{gap15}. The existing implementation uses Saucy~\cite{saucy} for graph isomorphism and Gibbs sampler as the base Markov chain.  We experiment on two versions each of two different domains, with context variables pre-specified. We next describe our domains.

\subsection{Domains and Methodology}

\begin{figure*}
\framebox{
\subfigure[]{
{\includegraphics[width=0.24\textwidth]{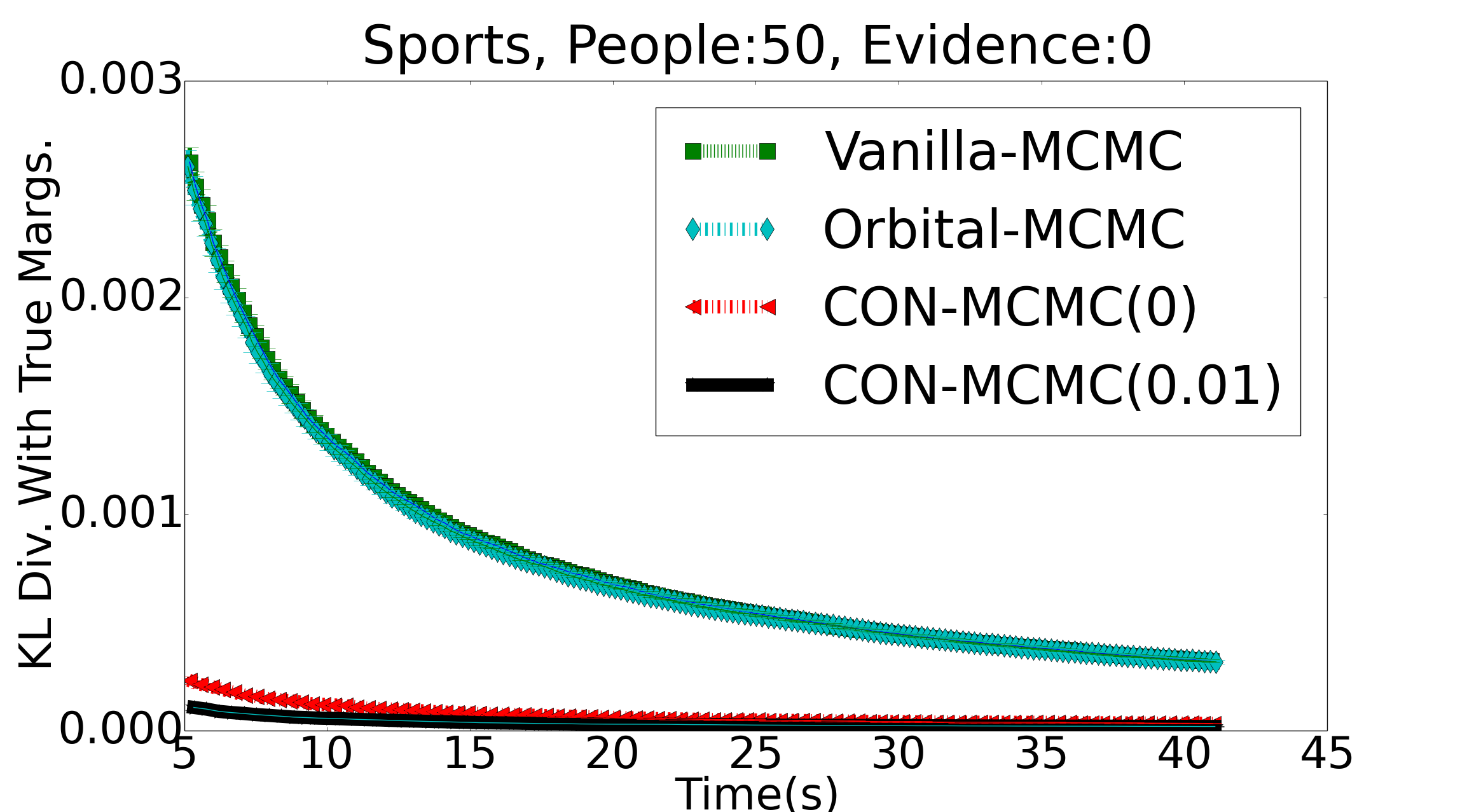}}
{\includegraphics[width=0.24\textwidth]{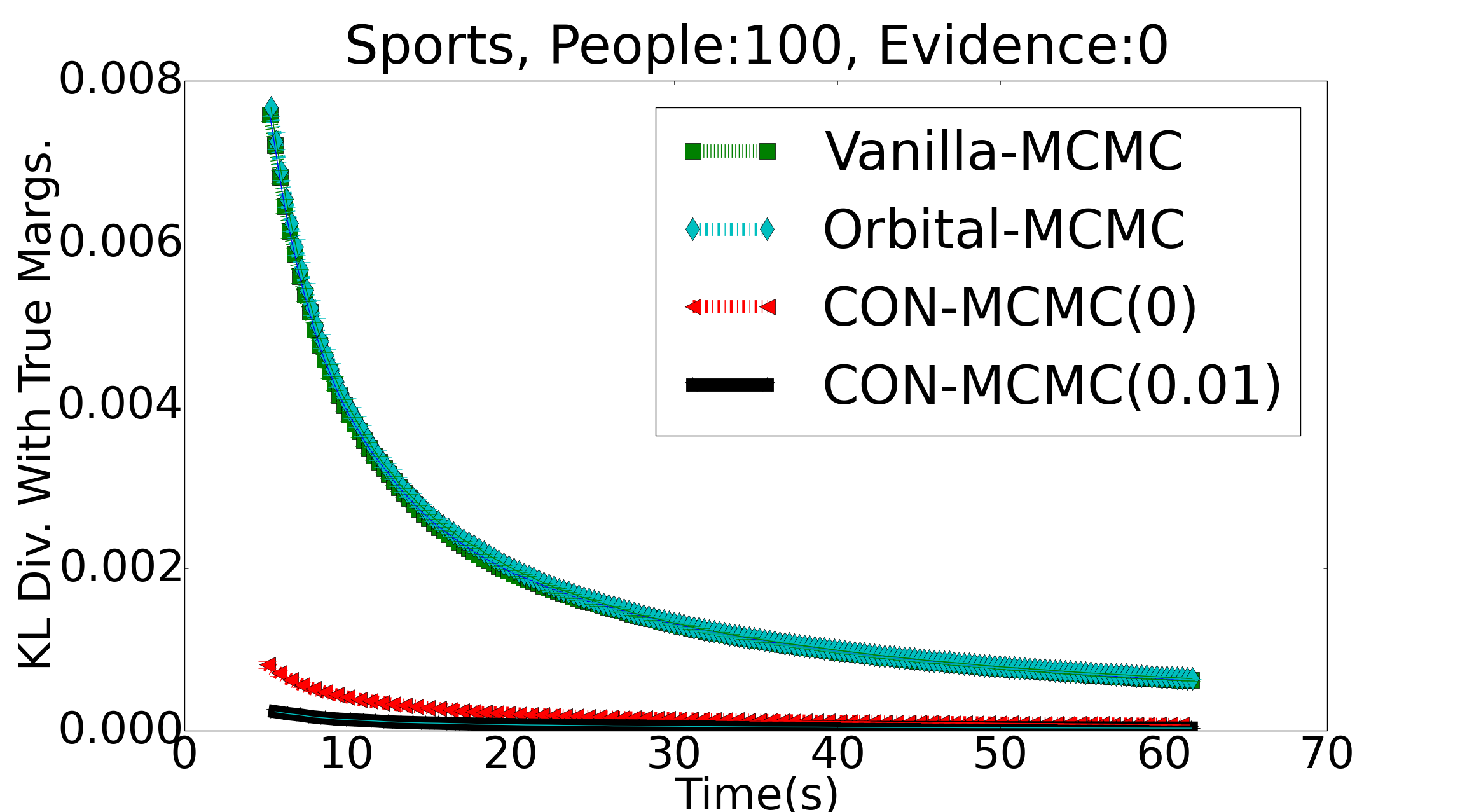}}
{\includegraphics[width=0.24\textwidth]{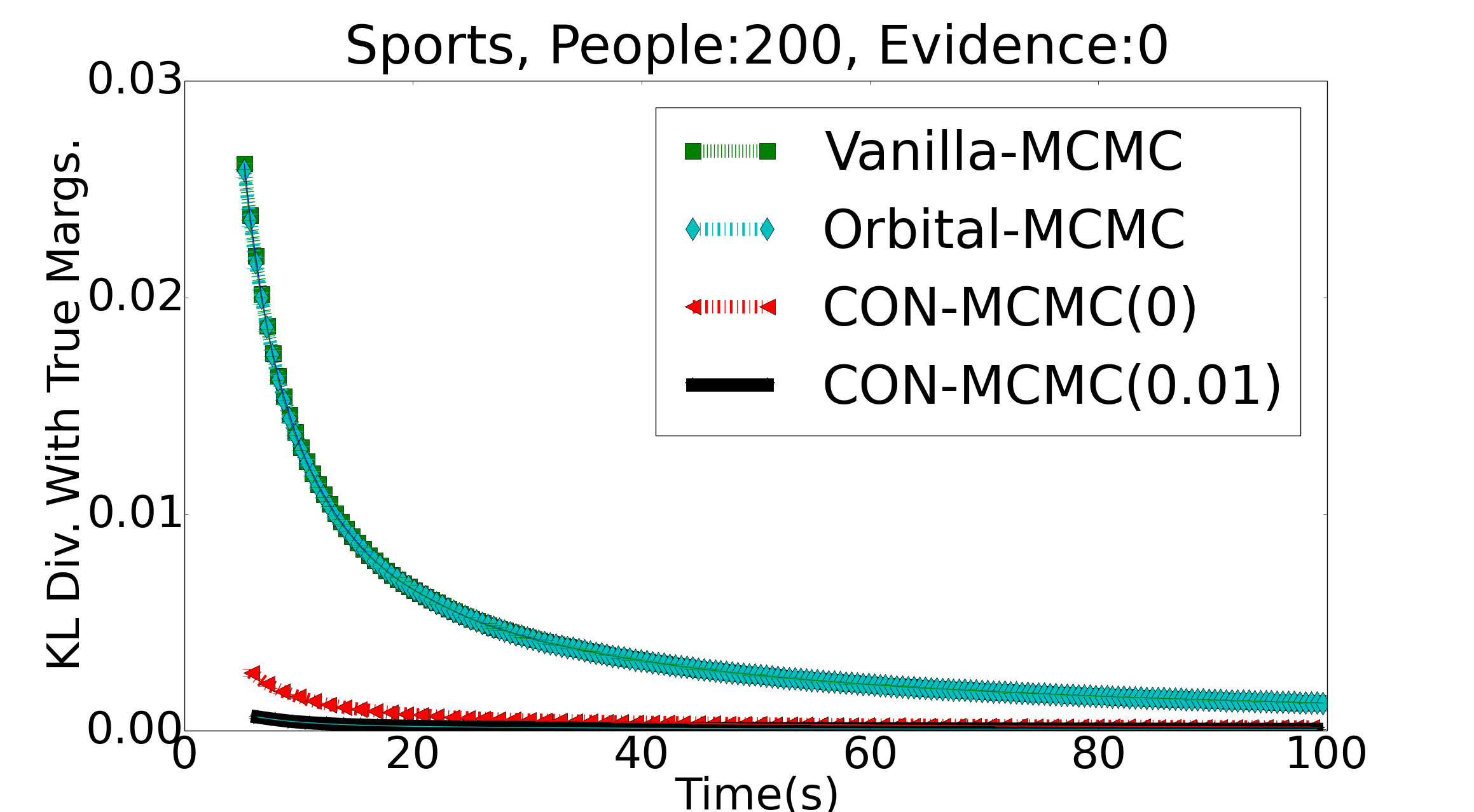}}}\label{fig:sports_size}}
\subfigure[]{
{\includegraphics[width=0.23\textwidth]{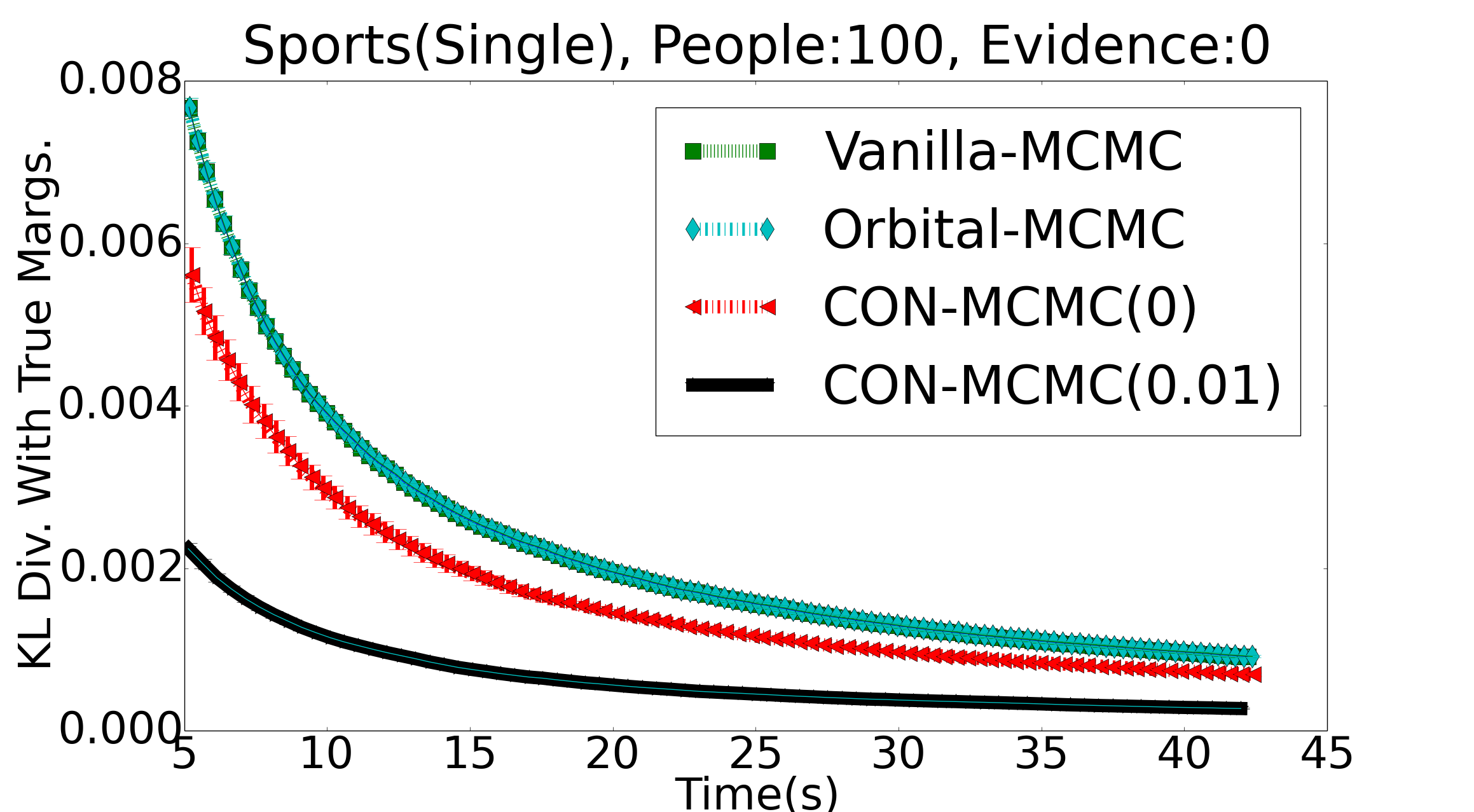}}\label{fig:sports_one_sided}}

\framebox{
\subfigure[]{
{\includegraphics[width=0.24\textwidth]{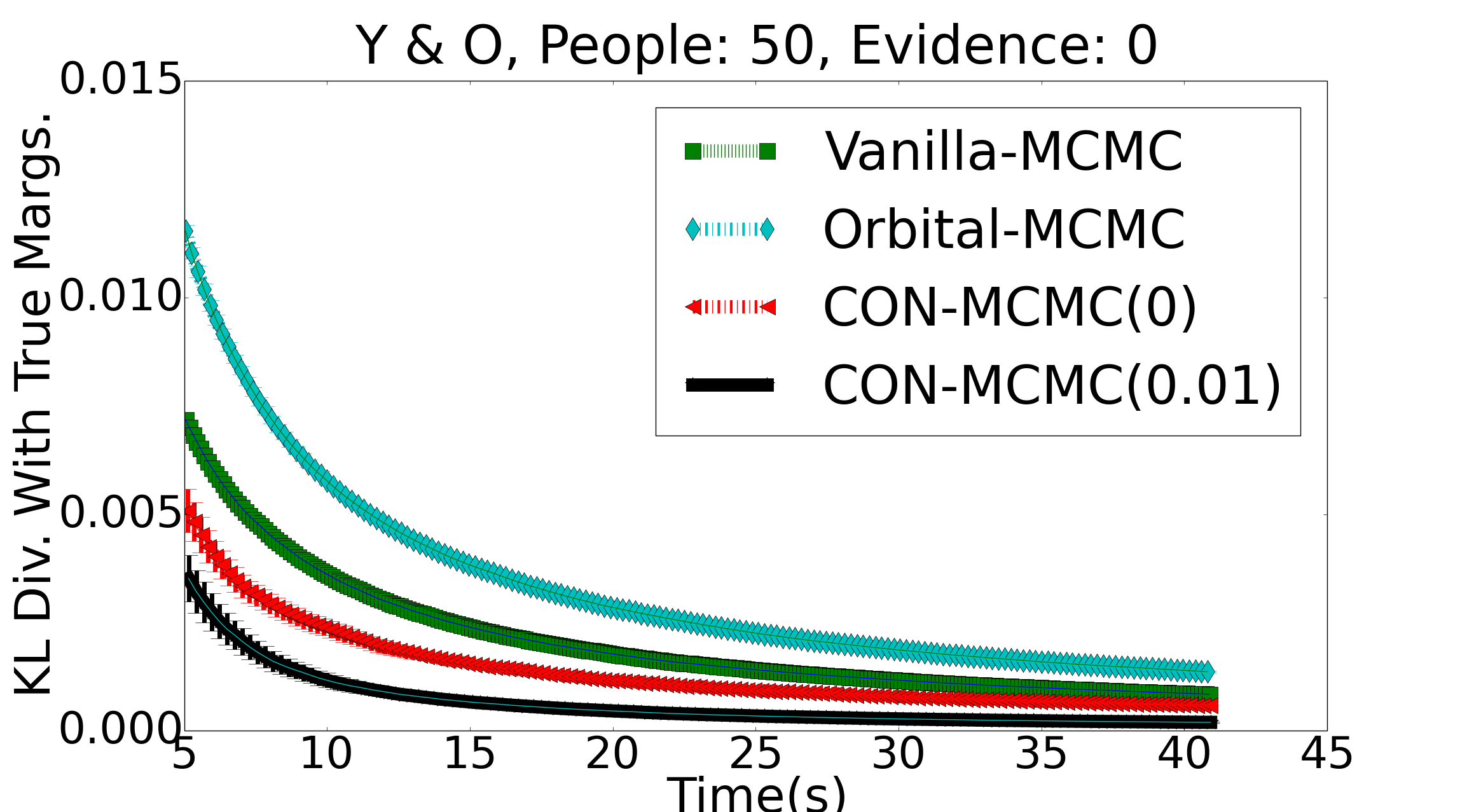}}
{\includegraphics[width=0.24\textwidth]{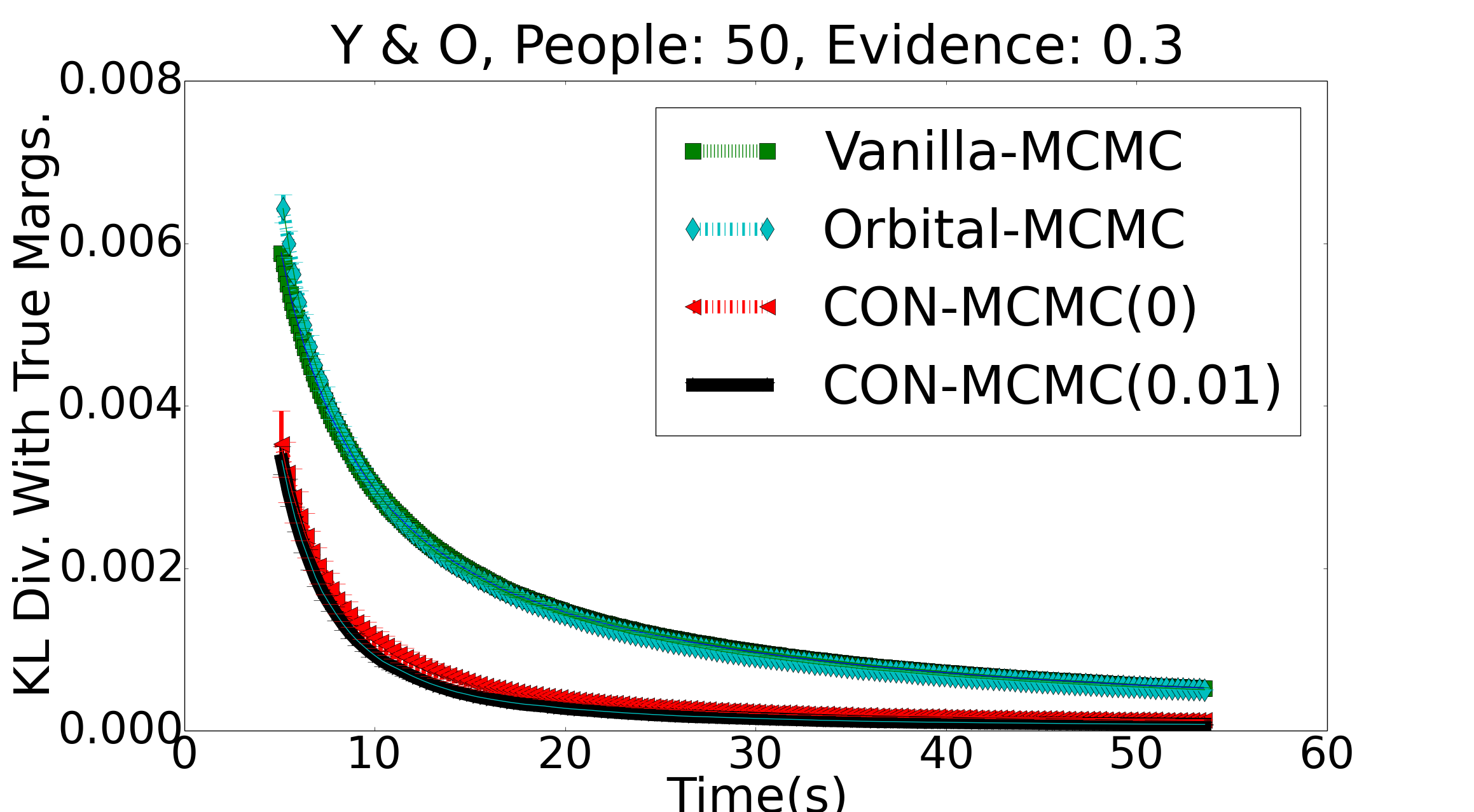}}
{\includegraphics[width=0.24\textwidth]{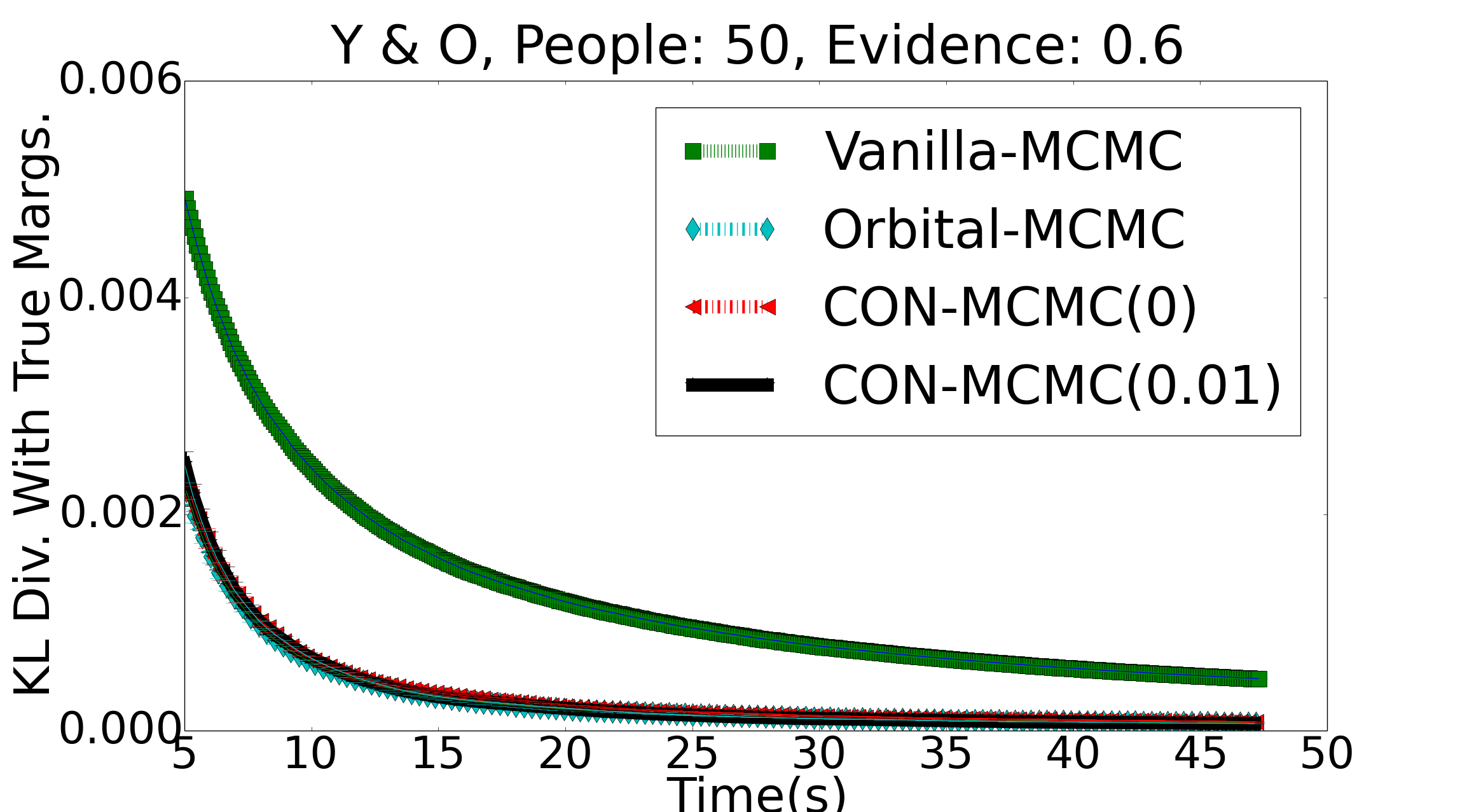}}}\label{fig:fs_evidence}}
\subfigure[]{
{\includegraphics[width=0.23\textwidth]{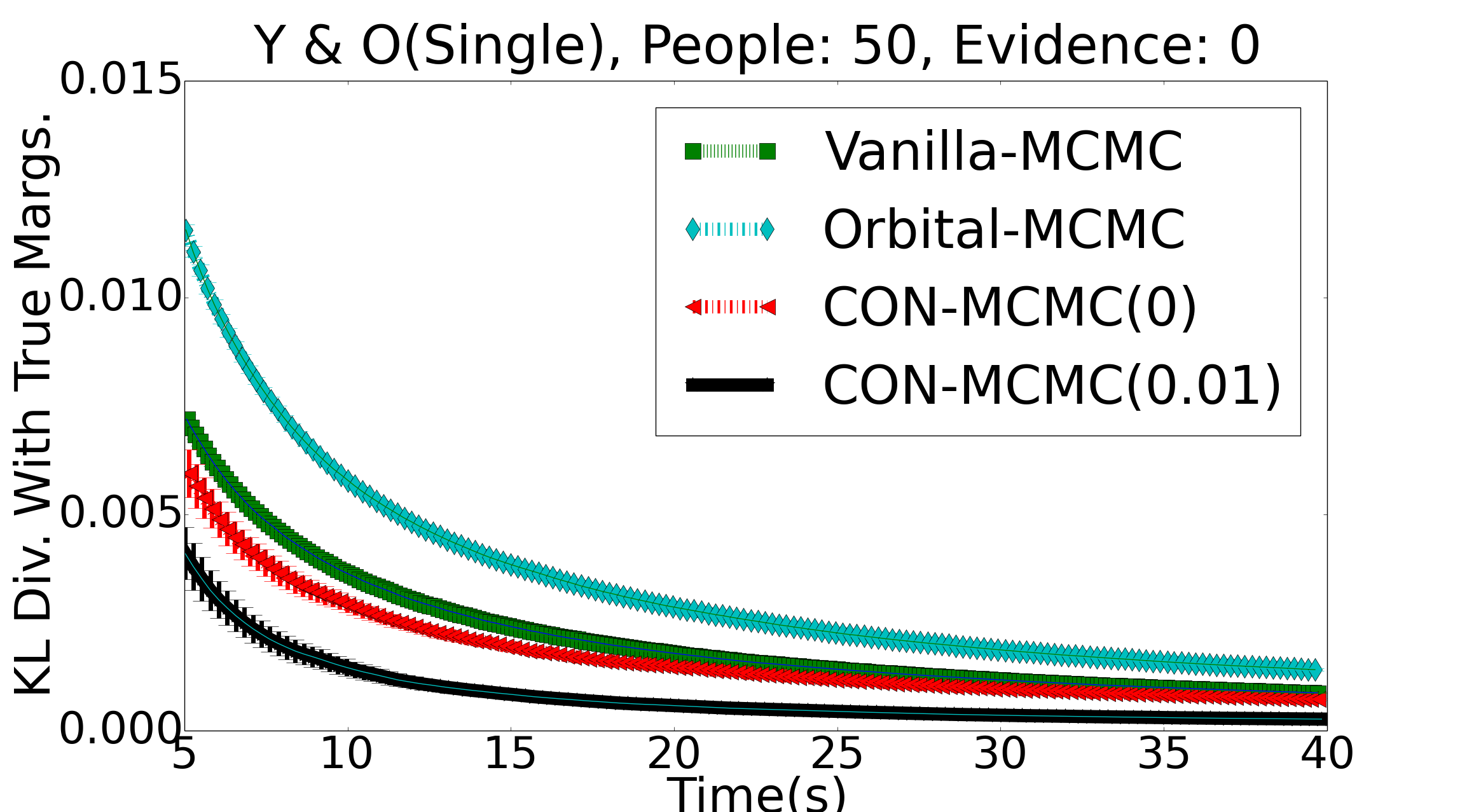}}\label{fig:fs_one_sided}}
\caption{{\bf (a)} \conmcmc\ effectiveness increases tremendously with increasing domain sizes. Note that y-axes are on different scales. {\bf (c)} New orbital symmetries are created with increasing evidence, leading to improved performance of Orbital MCMC. {\bf (b, d)} Curves for Sports Network (Single) and Y \& O (Single) respectively -- \conmcmc(0.01) performs the best and vastly outperforms \conmcmc(0).}
\label{fig:results}
\end{figure*}

\noindent{\bf Sports Network: } This Markov network models a group of students who may enter a future sport league, which could be for one of two sports, badminton or tennis (modeled as the variable $Sport$). Each student belongs to one of the dorms on campus. The league accepts both singles as well as doubles entries. For each student $X$, the domain has a variable for playing singles, $S_X$. For each pairs of students $X$, $Y$ coming from the same dorm, we have a variable indicating that they will play doubles together, $D_{XY}$. Multiple students (in the same dorm) train together in training groups, which are different for the two sports. A student's participation in the league for a given sport is (jointly) influenced by the participation of other students in her training group for that sport. Moreover, if two students decide to play singles, it increases the probability that they may also team up to play doubles independent of their training groups. In this domain, different subsets of students in a dorm (based on their training groups) are symmetrical to each other depending upon $Sport$, which becomes a natural choice for the context. In our experiments, we use training groups of $5$ students and dorms with $25$ students each.
% and total student population of $100$.

\vspace{0.5ex}
\noindent {\bf Young and Old: } This domain is modeled as an MLN and is an extension of the Friends and Smokers (FS) \cite{singla&domingos08} network. $Y\&O$ has a propositional variable $IsYoung$ determining whether we are dealing with a population of youngsters or older folks. For every person $X$ in the domain, we have predicates $Smokes(X)$, $Cancer(X)$ and $EatsOut(X)$. We also have the predicate $Friends(X,Y)$ for every pair of persons. We have rules stating that young persons are more likely to smoke and older people are less likely to smoke. Similarly, we have rules stating that young people are more likely to eat out and old people are less likely to eat out. When the population is young, everyone has the same weight for the smoking rule and slightly different weight (sampled from a Gaussian) for eating out. When the population is old, everyone has a slightly different weight (again sampled from a Gaussian) for smoking and the same weight for eating out. As in the original FS, we have rules stating smoking causes cancer and friends have similar smoking habits. We also have rules stating that cancer and friends variables have low prior probabilities. In this domain, smoking, cancer and friends variables are symmetric to each other when population is young, whereas all eating out variables are symmetric when the population is old. Clearly, $IsYoung$ is a natural choice for context in this domain.

\vspace{0.5ex}

An important property of both these domains is that different contextual symmetries exist for {\em both} assignments of the respective context variables. To test the robustness of \conmcmc\ we further modify these domains so that contextual symmetries exist only on {\em one} of the two assignments of context variable. In $Y\&O$ (Single), we give (slightly) different weights to $EatsOut(X)$ variables when $IsYoung$ is false, i.e., symmetries exist only when $IsYoung$ is true. In Sports Network (Single), $S_X$ variables involved in a training group are symmetric only for tennis; for badminton, each $S_X$  in a group behaves (slightly) differently. We refer to these two variations as the {\it Single} side versions of the original domains.

For these four domains, we plot run time {\em vs.} the KL-divergence between approximate marginal probabilities computed by each algorithm and the true marginals.\footnote{computed by running a Gibbs sampler for sufficiently long time.} For both Orbital MCMC and \conmcmc, the time to compute symmetries is included in the run time. For each problem we run 20 iterations of each algorithm and take the mean of the marginals to reduce variance of the measurements. We also plot the 95\% confidence intervals. We show \conmcmc\ results for $\alpha=0$ and $0.01$, which was chosen based on performance on smaller problem sizes. We perform various control experiments by varying the size of domains, amount of available evidence, marginal posterior probability of the context variable and the value of $\alpha$ parameter. All the experiments are run on a quad-core Intel i-7 processor.

\subsection{Results}
\label{sec:results}
Figures \ref{fig:results} and \ref{fig:results_posterior} show the representative graphs across multiple domains and varying experimental conditions. We find that \conmcmc(0.01) almost always performs the best or at par with the best of other three algorithms. \conmcmc(0) usually performs better than Gibbs and Orbital MCMC, but its performance can be closer to Gibbs or \conmcmc($\alpha$) depending upon the experimental setting. Orbital MCMC does not usually offer much advantage over Gibbs, primarily because these domains don't have many orbital symmetries. For Sports Network, there are no orbital symmetries at all; Orbital MCMC avoids the overhead of the orbital move and performs at par with Gibbs. For $Y\&O$ Orbital MCMC finds a few symmetries, which don't particularly help in reducing mixing time. However, it still incurs the overhead of orbital moves, leading to a significantly worse performance compared to Gibbs.

\vspace{1ex}
\noindent{\bf Variation with Domain Size: } Figure \ref{fig:results}(a) compares the algorithms as we increase the domain size for the Sports network from $50$ to $200$ students. The overall trends remain similar, i.e., \conmcmc\ algorithms outperform Gibbs and Orbital MCMC by huge margins. A closer look reveals that the y-axes are at different scales for the three curves -- the relative edge of \conmcmc\ algorithms increases substantially with larger domain sizes. 

\vspace{1ex}
\noindent{\bf Variation with Amount of Evidence: } Figure \ref{fig:results}(c) compares the performance of the algorithms as we vary the amount of (random) evidence available from $0\%$ to $60\%$ in the $Y\&O$ domain on predicates other than $Friends(X,Y)$ using a domain size of $50$. As earlier, \conmcmc\ algorithms outperform others. We observe that the relative gain of \conmcmc\ algorithms with Orbital MCMC decreases with increasing evidence (for 30\% evidence Orbital MCMC overlaps with Gibbs, for 60\% evidence, Orbital MCMC overlaps with \conmcmc). We believe that this is due the fact that more evidence tends to disconnect the network introducing additional symmetries which can be exploited by Orbital MCMC. Nevertheless, \conmcmc\ algorithms 
perform at least as well as Orbital MCMC for all values of evidence that we tested on.

\begin{figure*}
\framebox{
\subfigure[]{
{\includegraphics[width=0.24\textwidth]{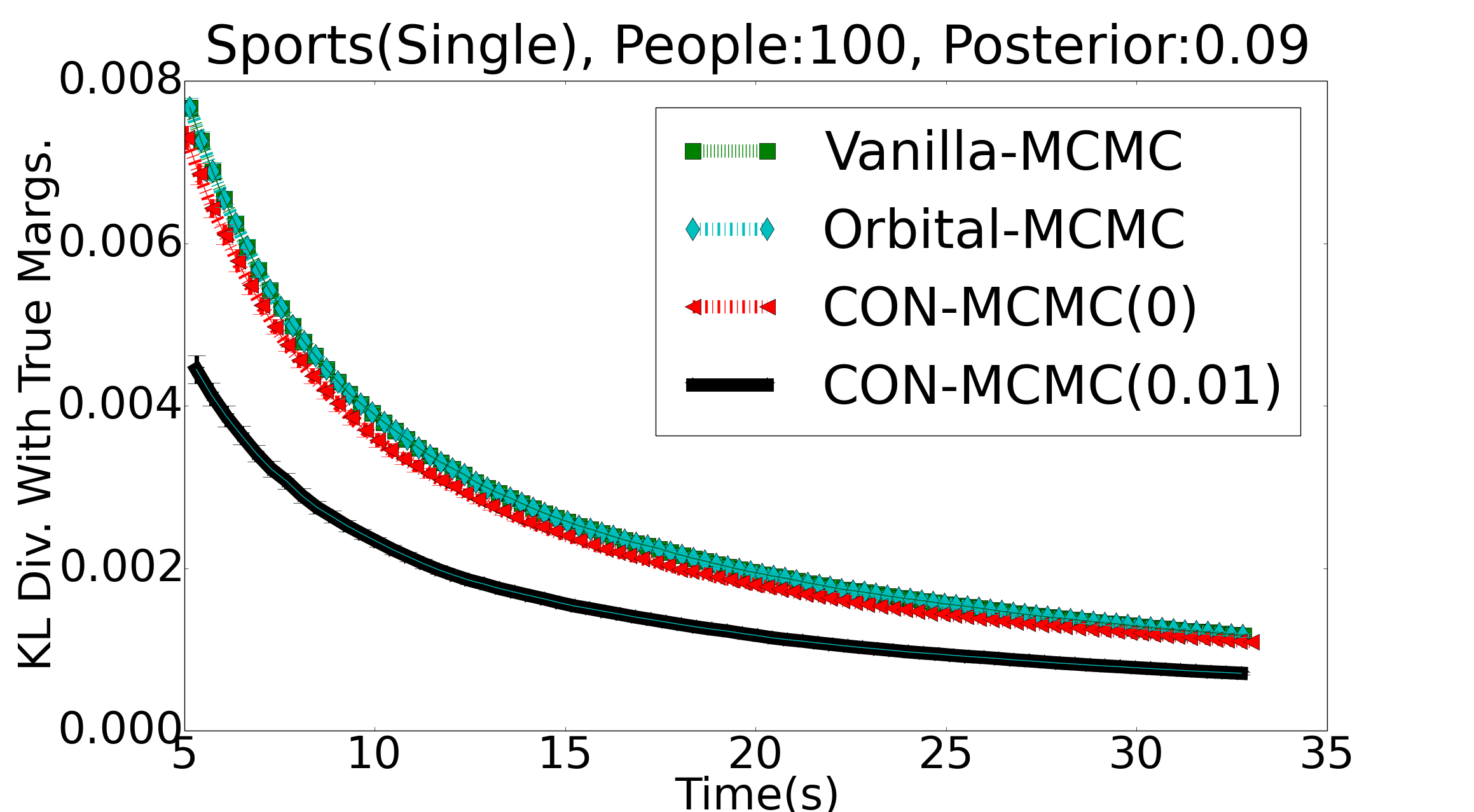}}
{\includegraphics[width=0.24\textwidth]{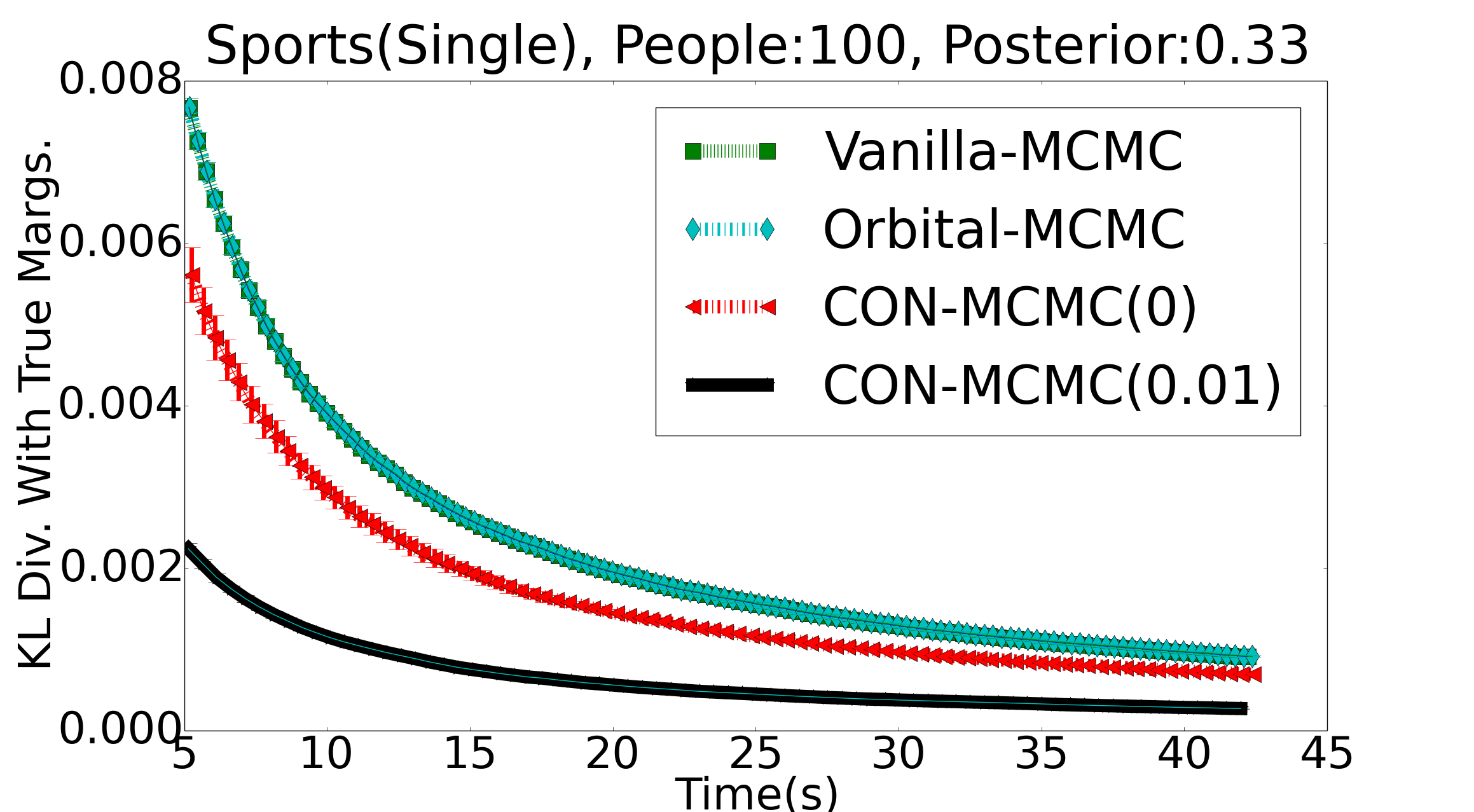}}
{\includegraphics[width=0.24\textwidth]{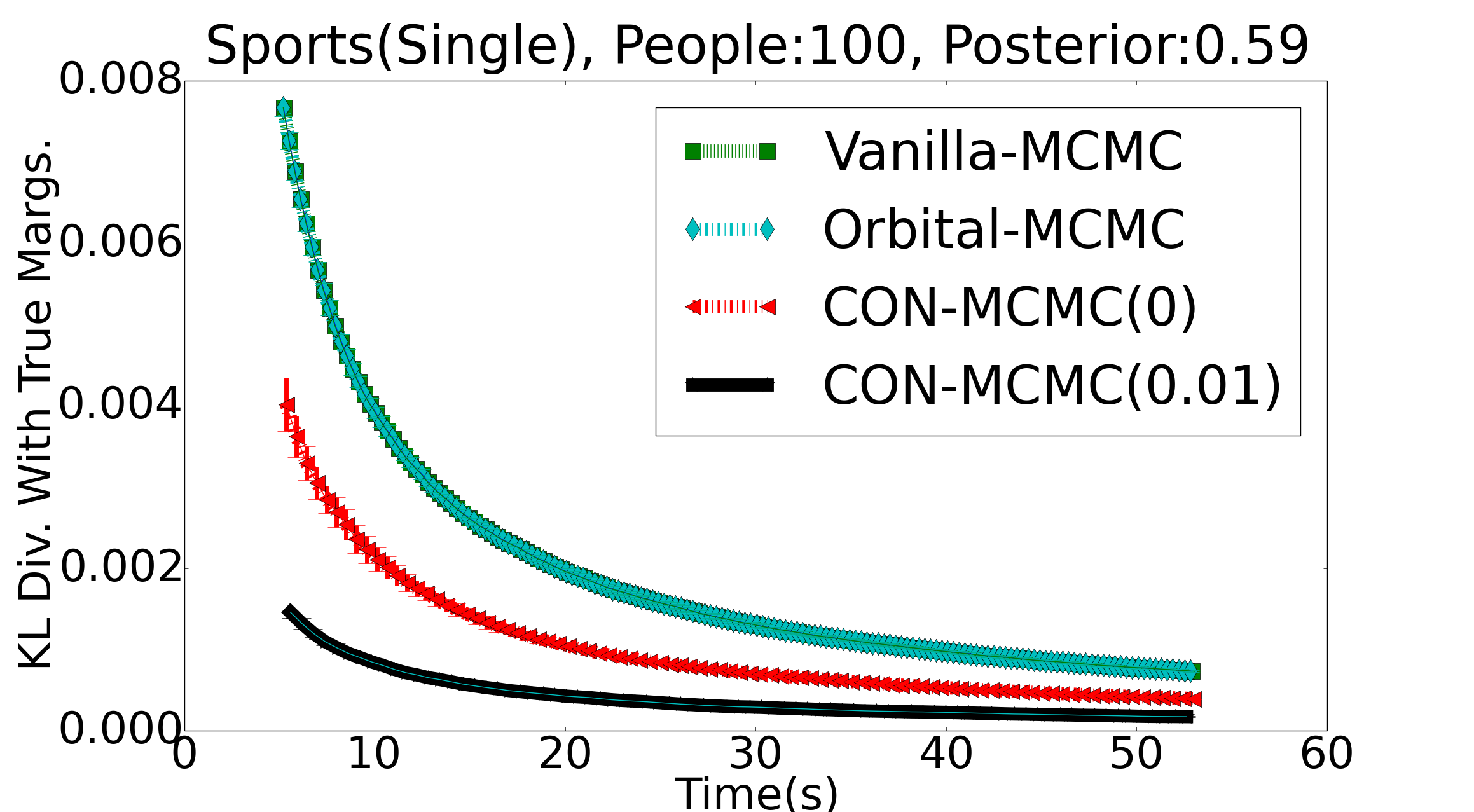}}
{\includegraphics[width=0.23\textwidth]{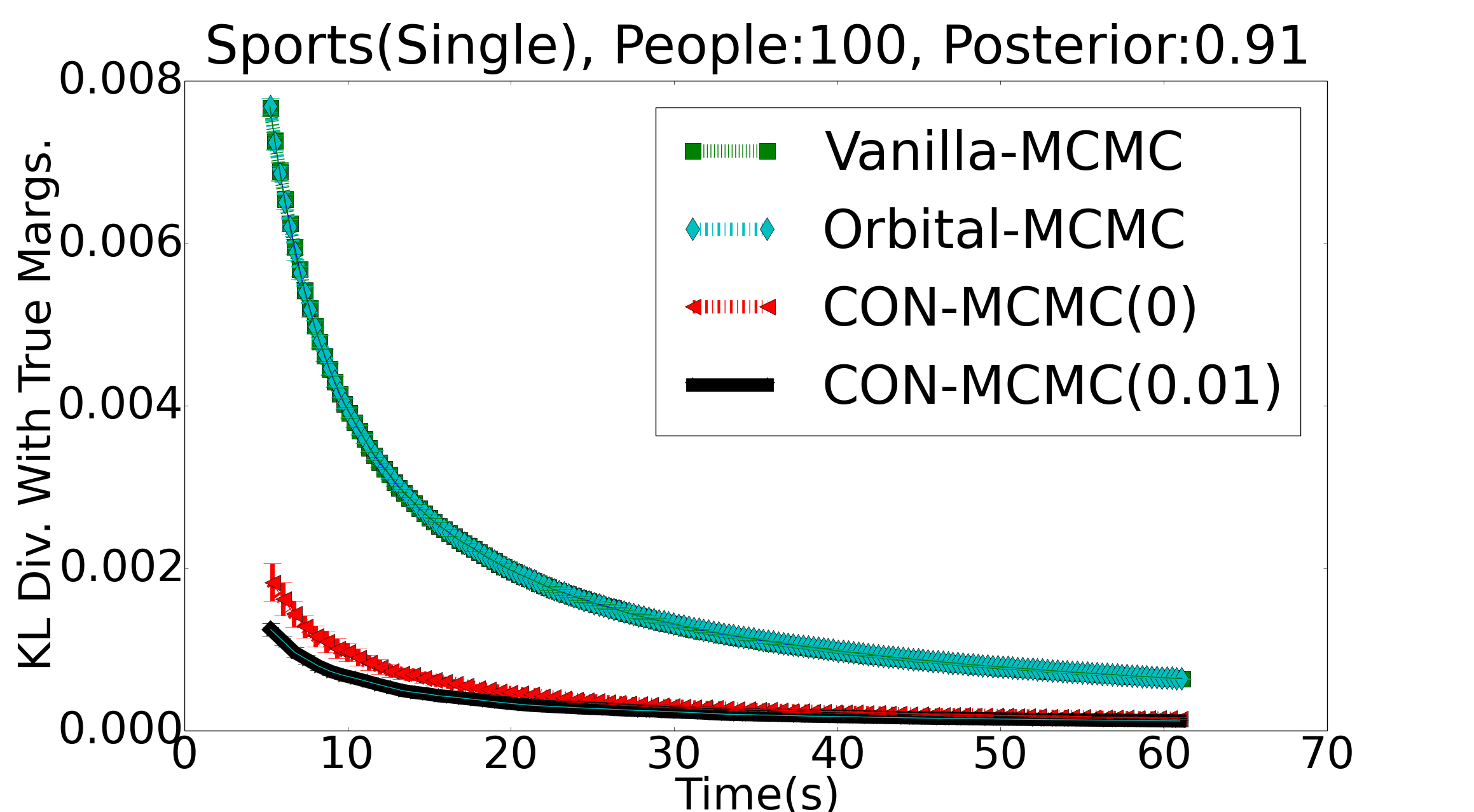}}}}

\caption{ \conmcmc\ effectiveness increases in Single Side Symmetry cases as we increase the marginal of context variable to the side having symmetry from 0.09 to 0.91. \conmcmc(0.01) provides significant gains even at very low posterior values. \conmcmc(0) performance improves with increase in the marginal.
}
\label{fig:results_posterior}
\end{figure*}

\begin{figure}
\hspace{2mm}
\includegraphics[width=0.23\textwidth]{{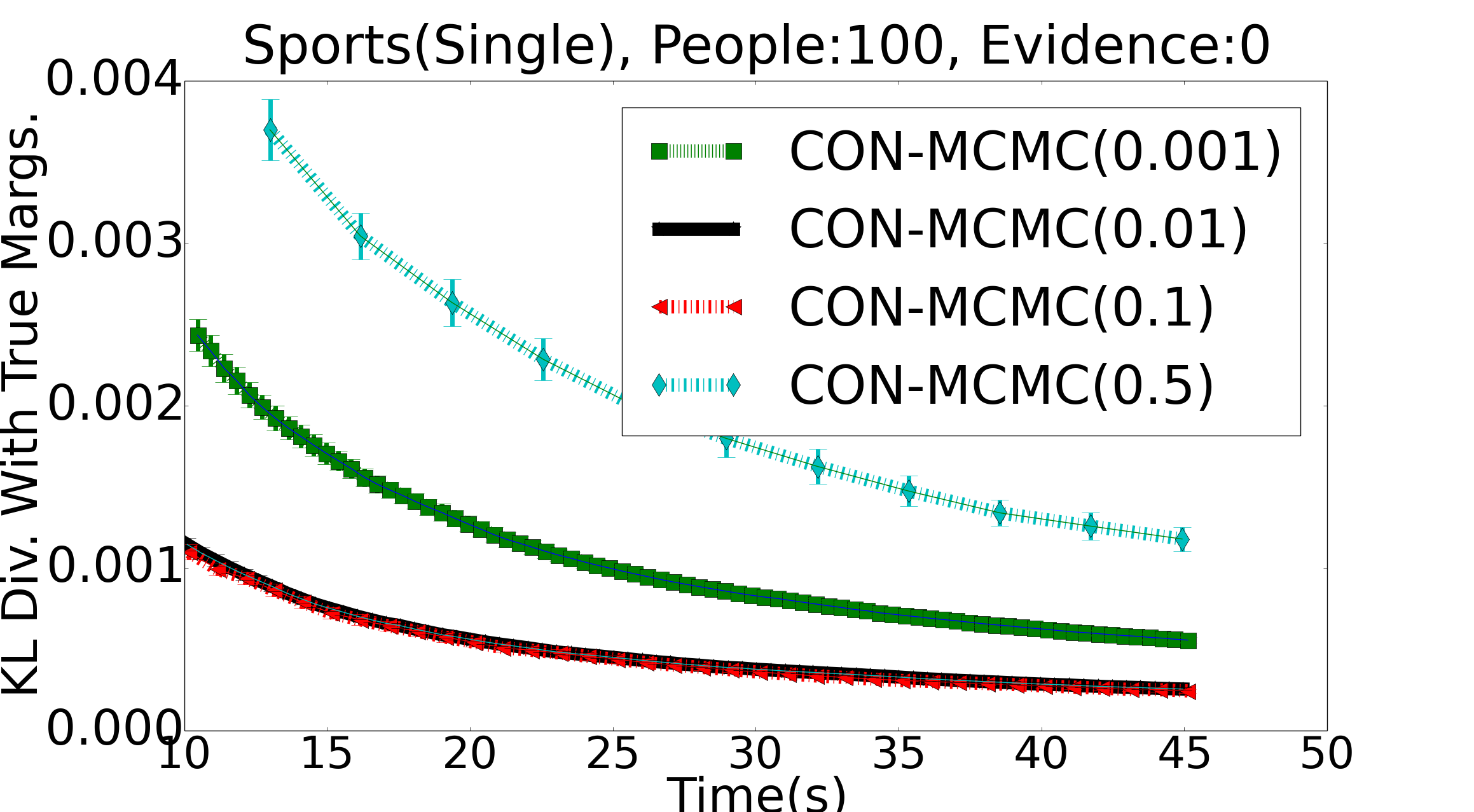}}\label{fig:alpha-sports}
\includegraphics[width=0.23\textwidth]{{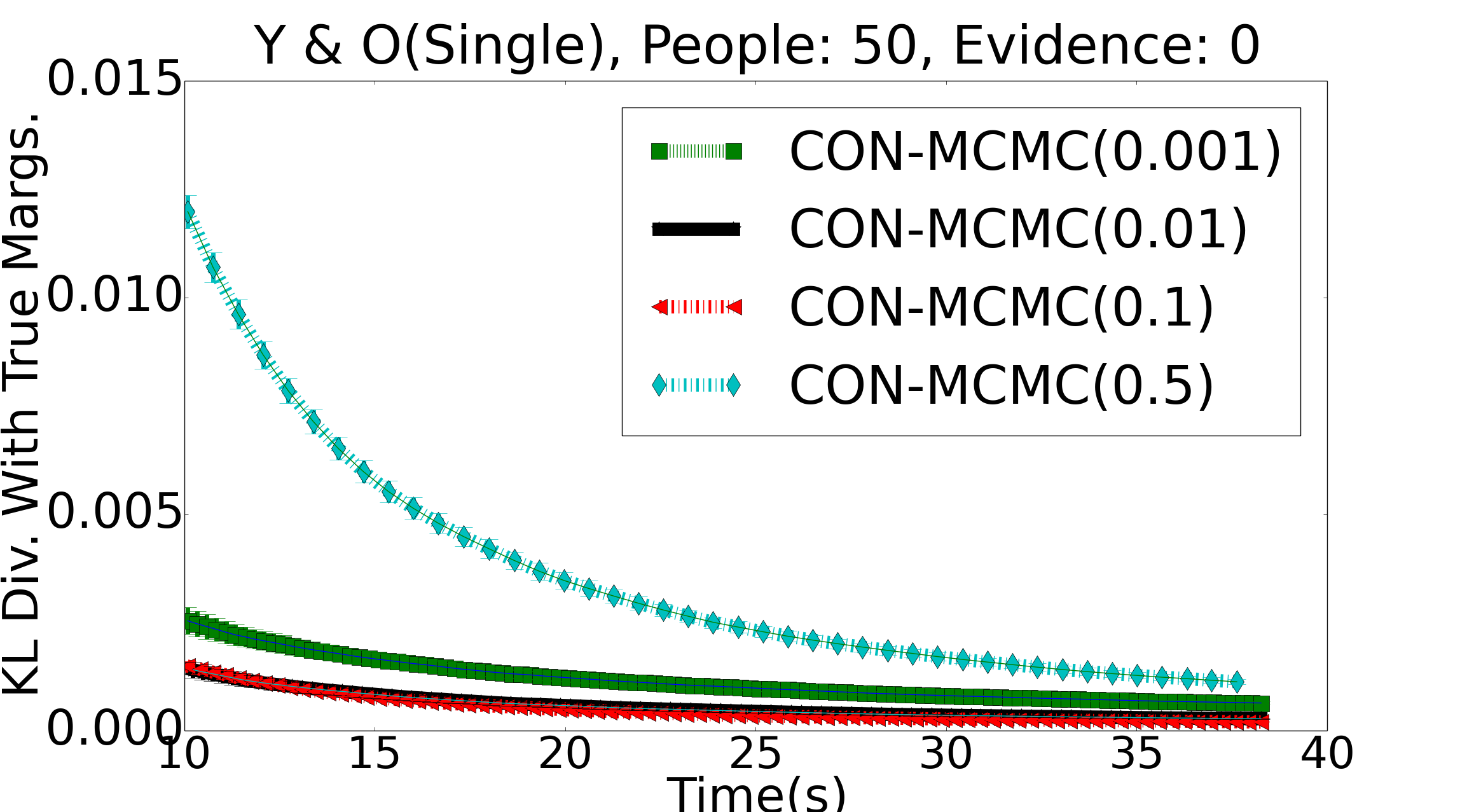}}
\label{fig:alpha-fs}
\vspace{-2ex}
\caption{$\alpha$=0.01 and $\alpha$=0.1 work best across both domains. Very high as well as very low values of $\alpha$ lead to poor performance.}\label{fig:alpha}
%\vspace{-2ex}
\end{figure}

\vspace{1ex}
\noindent{\bf Variation across Versions of a Domain: }
Figure~\ref{fig:results}(b) and \ref{fig:results}(d) show the plots for the {\it Single} side versions of Sports network and $Y\&O$, respectively. We observe a significant difference in the performance of the two \conmcmc\ algorithms. The reason is subtle. Since symmetries exist only on one side, that side mixes quickly for \conmcmc(0); however, the other side does not mix as well, because of lack of symmetries. \conmcmc($\alpha$) for $\alpha>0$ mitigates this by upsampling the flip of the context variable. This enables the rapid mixing on symmetry side to regularly influence the non-symmetry side (via Gibbs move), which leads to a faster mixing on that side too. Nevertheless, \conmcmc(0) is still able to outperform both Gibbs sampling as well as Orbital MCMC by exploiting the single sided symmetry. The posterior of symmetry side is 0.33 in Sports network (Figure~\ref{fig:results}(b)) and 0.37 in $Y\&O$ (Figure~\ref{fig:results}(d)).

We also observe in the first graph of Figure \ref{fig:results}(c), that \conmcmc(0) performs somewhat worse than \conmcmc(0.01). We believe that the reason for performance in this two-sided symmetry domain is similar to the single-sided case. In $Y\&O$, when $IsYoung=$true, substantial symmetries may exist due to smoking, cancer and friends variables. However, on the other side, the symmetries are far less (only for eating out variables). This implies that \conmcmc(0) will have much faster mixing on one side, but not on the other. On the other hand, \conmcmc($0.01$) will upsample context variable flips and allow the stronger symmetry side to influence the other. In general, \conmcmc($\alpha$) performance is highly robust to varieties of symmetric and asymmetric domains.

\vspace{1ex}
\noindent{\textbf{Variation with Posterior of Context Variable:}} 
We investigate performance on Single-sided domains further by varying the posterior marginal probability of the context variable. Figure \ref{fig:results_posterior} shows the results for Sports network (Single) with marginal probability of $Sport=$ tennis varying from $0.09$ to $0.91$. Note that $Sport=$ tennis side is the side where symmetries exist. 

The graphs show an interesting trend. Even for very low marginals, \conmcmc(0.01) is able to benefit from one sided symmetries. Since the marginal is low we expect any MCMC algorithm to spend most of its time on the non-symmetry side. However, \conmcmc(0.01) will still go back and forth several times between two sides; each flip to symmetry side and back will help in potentially reaching a different region of the state space leading to better mixing on the non-symmetry side.

Not surprisingly, \conmcmc(0) does not perform as well for low marginals -- it does not get to switch contexts as often, and ends up mixing slowly on the important, non symmetry side. As marginal of the context variable increases, the relative performance of \conmcmc(0) improves substantially. As marginal becomes high ($0.91$), both \conmcmc\ samplers end up sampling mostly on the symmetry side, and can reap benefits of symmetries similarly. We also conduct these experiments for the $Y\&O$ domain and  observe a very similar behavior.

\vspace{1ex}
\noindent
\textbf{Variation with $\alpha$ Parameter:} Figure \ref{fig:alpha} shows the performance of \conmcmc($\alpha$) for different values of $\alpha$ in the range 0.001 to 0.5 for both Sports network (single) and $Y\&O$ (single) domains. Our algorithm is fairly robust for values of $\alpha$ between $0.01$ and $0.1$. Its performance starts to degrade for very low as well as very high values of $\alpha$. For very low values of $\alpha$, algorithm's behavior approaches that of \conmcmc(0). For very high values of $\alpha$, the algorithm spends too much time flipping the context variable and not enough time exploring the state space, resulting in poor performance.

Overall, we conclude that \conmcmc(0.01) is robust to various experimental settings and obtains the best results significantly outperforming Orbital MCMC and Gibbs. This underscores the importance of our contextual symmetry framework for probabilistic inference. 

\section{Discussion and Future Work}
While our work extends the capability of lifted inference to a wider range of settings, it also raises important questions. In many cases, the set $V$ of context variables is known from domain knowledge or domain description especially in relational models. An open question is how to automatically compute a good set $V$, since trying all possible sets can be prohibitive. We have designed a heuristic approach that greedily chooses the most useful context variable every iteration and adds it to the context set. It uses a few initial rounds of the color passing algorithm~\cite{kersting&al09} to approximate the amount of additional symmetry obtained by making a variable part of the context. More experiments are needed to assess the effectiveness of our approach.

Another important observation is that the set of contextual symmetries may not monotonically increase with increasing context size. This may happen if additional context variables break existing symmetries, since context variables are forced to undergo identity mapping. Then, how do we design algorithms so that their effectiveness monotonically increases with larger contexts in all cases? This is an important direction for future work.

Another question concerns the robustness of performance of symmetry-based inference algorithms. Over the course of our experiments, we tested our algorithms on several domain variations. While in most cases \conmcmc(0.01) and \conmcmc(0) performed much better than Gibbs, in rare cases, the performance was worse too. Further investigations revealed two main sources of lower performance. 

The first and more prominent cause is the trade-off between mixing speed and sampling time. Because all symmetry-based algorithms run an expensive product replacement algorithm \cite{pak00} to sample from an orbit, next samples for \conmcmc\ (and Orbital MCMC) are generated much slower than Gibbs. In domains where symmetries are prevalent, this slower sampling is mitigated by rapid mixing, but in other domains, it could result in a worse performance. An intelligent wrapper that guesses whether to exploit symmetries or not in a given domain will be crucial for developing a robust inference algorithm. The second reason for lower performance is subtle. \conmcmc($\alpha$) is able to exploit contextual symmetries (even single-sided)  in a wide variety of settings, but in one situation it can lose to other algorithms. This happens when the context variable has a huge Markov blanket, so much so that one Gibbs move that flips the context variable becomes overbearingly costly. Since \conmcmc($\alpha$) upsamples flips of context variables, this can cause significant loss to overall performance, even though the mixing is much faster with respect to number of samples.

Another observation relates to the effect of evidence in a domain. Evidence can both help and hurt symmetries in an inference problem. In some cases, evidence can break existing symmetries and reduce the relative gain of symmetry-based algorithms. In other cases, evidence can break edges and create new symmetries and help them. While in our experiments, we didn't find \conmcmc(0.01) to be ever worse than Gibbs due to additional evidence, such pathological cases can be constructed.

It would be interesting to see how algorithms other than MCMC can benefit from our contextual symmetry framework. In the future, we would also like to explore approximate contextual symmetries that could make our contribution applicable to several other domains, where exact contextual symmetries cannot be found. We would also like to theoretically analyze the mixing time of \conmcmc.

\section{Related Work}
\label{sec:related}

Some papers have discussed methods for computing symmetries under a given evidence \cite{broeck&darwiche13,venugopal&gogate14,kopp&al15}. As discussed in Section \ref{subsec:compute}, the algorithm for computing contextual symmetries is closely related to computing evidence-based symmetries. The main difference is in the way we use these symmetries for downstream inference.
%the nature of downstream use for inference. 

While our general notion of contextual symmetries is novel it has connections to a few recent works. The RockIt system \cite{noessner&al13} identifies contextual symmetries in a very special case in which the domain theory has a set of disjunctive clauses of a specific kind $g_i \vee c$ where each $g_i$ is a single literal (or its negation). For this setting, $c$ is a natural context and symmetries among $g_i$s can be exploited. RockIt does not provide any general notion beyond this special case. It constructs a reduced ILP for MAP inference instead of marginal inference, as in our case.

There is recent work on exploring connections between the concept of exchangeability of random variables and tractability of probabilistic inference \cite{niepert&broeck14}. Our contextual symmetries can be seen as a generalization of their conditional decomposability to conditional {\em partial} decomposability where the sufficient statistics are precisely the contextual orbits. Whereas Niepert and Van den Broeck \shortcite{niepert&broeck14} primarily focus on developing the theory for conditional decomposability, we propose and additionally connect this with the symmetries present in the structure of a graphical model. Further, unlike them we develop an algorithm to compute these conditional decompositions (contextual symmetries in our case) and show how they can be used in practice for efficient probabilistic inference.

As discussed in Section \ref{sec:intro}, our work builds upon the recent literature on lifted inference that pre-computes explicit domain symmetries using automorphism groups \cite{niepert12,bui&al13,broeck&niepert15} and exploits them for efficient inference. Our work is most closely related to Orbital MCMC \cite{niepert12}. Our experimental results shows the value of \conmcmc\ over Orbital MCMC that does not incorporate contextual symmetries.

Our contextual symmetries are also analogous to \emph{conditional symmetries} in constraint satisfaction problems (CSPs) \cite{gent2005-cp,walsh2006-cp,gent2007-cp}. CSP symmetries are called conditional if symmetry groups exist only in a sub-problem of the original CSP, i.e., in a CSP with one or more additional constraints. The CSP problem setting and their actual manifestation in algorithms are quite different from lifted inference, but their definition and use of conditional symmetries is in the same spirit as ours.

\section{Conclusions}
We present a novel framework for contextual symmetries in probabilistic graphical models. Contextual symmetries generalize and extend previous notions of orbital symmetry. Given any context, we can efficiently compute these symmetries by reducing it to the problem of colored graph isomorphism. While our framework is independent of any inference algorithm, we illustrate its applicability by proposing \conmcmc, an MCMC approach that exploits contextual symmetries. Our experiments on several domains validate the efficacy of \conmcmc, where it outperforms existing state-of-the-art techniques for symmetry-based MCMC by wide margins. Finally, we have released a reference implementation of \conmcmc\ for wider use by the research community.

\section*{Acknowledgements}
We are grateful to Mathias Niepert for sharing the code implementation of Orbital MCMC and for answering our queries on the code. We would also like to thank the anonymous reviewers for their comments and feedback. We thank Ritesh Noothigattu for discussions and comments on this research. 
Ankit Anand is supported by  TCS Research  Scholars Program.  Mausam and Parag Singla are supported by Visvesvaraya faculty research awards by Govt. of India. Mausam is also supported by Google and Bloomberg research awards.

\bibliographystyle{named}
\bibliography{all}

\end{document}